\documentclass[runningheads]{llncs}
\usepackage{amssymb,amscd,amsmath,amsfonts,xcolor}
\usepackage{graphicx,bbold,color}
\usepackage{bbold}
\usepackage{mathrsfs}
\usepackage{amstext,stmaryrd}
\newcommand{\Set}[1]{\{#1\}}
\newcommand{\mymod}[1]{\llbracket #1 \rrbracket}
\usepackage{esvect}

\newcommand{\POm}{{\bf 2^\Omega}}
\usepackage{float}
\newtheorem{fact}{Fact}

\usepackage{csquotes}

\newcommand{\ep}{\hfill$\Box$}
\newcommand{\Blp}{\mathscr{B}_{\lambda,P}}
\newcommand{\Blrp}{\mathscr{B}_{\lambda,\rlpsi}}

\newcommand{\amrand}[2]
{\leavevmode
	\marginpar
	[\raggedleft\scriptsize \leavevmode\normalcolor #1: #2]
	{\raggedright\scriptsize \leavevmode\normalcolor #1: #2}
}

\newcommand{\free}{\POm}

\newcommand{\dkl}{D_{K,L}}

\newcommand{\rlpsi}{R^\lambda_\psi}
\newcommand{\rlpsiuno}{R^\lambda_{\psi_1}}
\newcommand{\rlpsidos}{R^\lambda_{\psi_2}}
\newcommand{\Lang}{\mathcal{L}}
\newcommand{\lfilt}{\mathscr{F}}
\newcommand{\implica}{\rightarrow}
\newcommand\Omit[1]{}
\newcommand{\Bpar}[2]{\mathscr{B}_{{#1},{#2}}}

\usepackage[T1]{fontenc}
\usepackage{graphicx}
\begin{document}
\title{
On Lockean beliefs that are deductively closed and   minimal change}
%
%
\author{Tommaso Flaminio\inst{1}
\and
Lluis Godo\inst{1}
\and
Ram\'on Pino P\'erez\inst{2}
\and
 Lluis Subirana\inst{3}
}
\authorrunning{T. Flaminio et al.}
%
\institute{Artificial Intelligence Research Institute (IIIA - CSIC), Campus de la UAB, Bellaterra, Spain\\
\email{\{tommaso,godo\}@iiia.csic.es}\\
 \and
Centre de Recherche en Informatique de Lens (CRIL), Universit\'e d'Artois, UMR 8188,  Lens, France\\
\email{pinoperez@cril.fr}\\
\and
Department of Mathematics, University of Barcelona (UB), Barcelona, Spain\\ \email{lsubirana@yahoo.com}}
\maketitle              
\begin{abstract}
Within the formal setting of the Lockean thesis, an agent belief set is defined in terms of degrees of confidence and these  are described in probabilistic terms. This approach is of established interest, notwithstanding some limitations that make its use troublesome in some contexts, like, for instance, in belief change theory. Precisely,  Lockean belief sets are not generally closed under (classical) logical deduction. The aim of the present paper is twofold: on one side we provide two characterizations of those belief sets that are closed under classical logic deduction, and on the other we propose an approach to probabilistic update that allows us for a minimal revision of those beliefs, i.e., a revision 
obtained by making the fewest possible changes to the existing belief set while still accommodating the new information. In particular, we show how we can deductively close a belief set via a minimal revision.

\end{abstract}
\section{Introduction}
In \cite{Foley1992} Foley introduces what he calls the {\em Lockean thesis}, a philosophical principle suggesting that {\em beliefs} can be defined in terms of {\em confidence}, the latter being  formulated within the (subjective) probabilistic setting. More concretely, given the language $\mathcal{L}$ of classical propositional logic over finitely many variables, Foley claims that it is rational for an agent to believe in a statement, written $\varphi$ in $\mathcal{L}$, provided that the agent's subjective probability of $\varphi$ overcomes a certain threshold.
\begin{description}
\item[Lockean thesis:] For any $\varphi\in \mathcal{L}$, it is rational to believe in $\varphi$ provided that its probability $P(\varphi)$ overcomes a certain threshold $\lambda$.
\end{description}
Given a threshold $\lambda$ and a probability function $P$, one can hence define a {\em Lockean belief set},  as:
\begin{equation}\label{eqBelSets}
\mathscr{B}_{\lambda, P}=\{\varphi\in \mathcal{L}:  P(\varphi)\geq \lambda\}.
\end{equation}
Lockean belief sets represent the epistemic states of rational agents, and they have been considered as a basis for theories of probabilistic belief change. The idea of grounding belief change on Lockean belief sets is in fact not new and it has been proposed and already explored and investigated by others.  
Among them, it is worth recalling the following ones, whose ideas inspired the present paper:  \cite{ShFi19} by Shear and Fitelson; \cite{Hann1,Hann2} both by Hannson;  \cite{Leitgeb} by Leitgeb and \cite{CR19} by Cantwell and Rott who, similarly to what we also propose  in the present paper,  adopt (a variation of) Jeffrey conditionalization to deal with the belief revision process.

Note that the elementary observation that for a rational agent it is more reasonable to believe in a statement $\varphi$ rather than in its negation $\neg\varphi$, forces the threshold $\lambda$ to be  strictly above  $1/2$. This requirement immediately implies that belief sets defined as in (\ref{eqBelSets}) do not contain pairs of contradictory formulas, in the sense that it is not the case that, for any formula $\varphi$,  $\varphi$ and its negation $\neg\varphi$ both belong to a Lockean belief set $\mathscr{B}_{\lambda,P}$.

Although they are not contradictory in the above sense, Lockean belief sets lack some properties that are usually required  in the classical approaches to belief change as developed by  Alchourr\'on,  G\"ardenfors, and  Makinson \cite{AGM}; Katsuno and Mendelzon \cite{KM91} etc. One of them is {\em deductive closure}. Denoting by $\vdash$ the consequence relation of classical propositional logic, this principle reads as follows.
\begin{description}
\item[Deductive closure:] For any belief set $\mathscr{B}$, if $\mathscr{B}\vdash\varphi$, then $\varphi\in \mathscr{B}$.
\end{description} 
As we will see in a while, the deductive closure for a Lockean belief set is equivalent to the apparently weaker principle of  {\em conjunctive closure}.

\begin{description}
\item[Conjunctive closure:] For any belief set $\mathscr{B}$, if $\varphi_1\in \mathscr{B}, \ldots, \varphi_n\in \mathscr{B}$, then $\varphi_1 \wedge \ldots \wedge\varphi_n\in \mathscr{B}$.
\end{description}
The above principle of conjunctive closure  has been  studied in the literature (see \cite{BCF} and \cite{VbHW}) and it is involved in the discussions concerning some  paradoxical situations, the best known being Kyburg's lottery paradox \cite{Kyburg67} and Makinson's preface paradox \cite{Makinson}.

Our current research line aims at understanding how probabilistic belief change theory can be approached in a way that is alternative to those mentioned above and precisely by employing variants of Jeffrey's conditionalization that allow a principle of minimal change. These ideas will be  discussed in Section \ref{sec:5}, while the core of this article is to present results that tell us under which conditions the key property of deductive closure can be safely assumed to hold for Lockean belief sets.

More precisely, we will present two main results, both of which give a characterization for those probability functions $P$ ensuring the existence of suitable thresholds $\lambda$ for which $\Blp$ is closed under logical deduction. These 
are 
the main results of Section \ref{sec:3} and Section \ref{sec:4}, respectively. Then, in Section \ref{sec:5} we discuss how to revise a Lockean belief coherently with an intuitive principle of minimal change. Interestingly, our minimal change desiderata are met by a method that revises a prior probability that is closely related to Jeffrey conditionalization. In the same Section \ref{sec:5} we will show that, indeed, this revision method is {\em minimal} in a sense that will be made clear there. Also in this  in  section we  characterize the condition under which a (non necessarily deductively closed) Lockean belief set becomes deductively closed once revised by a suitable formula. Related works will be discussed in Section \ref{Sec:related}. 
We end this paper with some remarks and suggestions for future work that will be presented in Section \ref{sec:7}. In Section \ref{sec:2} we recall basic logical and algebraic notion and to fix our notation.

\section{Preliminaries}\label{sec:2}
The basic ground of our investigation is classical propositional logic (CPL) and our language $\mathcal{L}$, up to redundancy, is built from a finite set of propositional variables, say $x_1,\ldots, x_n$, connectives $\wedge, \vee, \neg$ for {\em conjunction}, {\em disjunction},  and {\em negation} respectively, and constants $\top$, $\bot$ for {\em true} and {\em false}.  
Formulas in that language will be denoted by lower case Greek letters $\varphi,\psi,\ldots$ with possible subscripts.  The connective of {\em implication} is defined as $\varphi\to\psi=\neg \varphi\vee\psi$, and {\em double implication} is  $\varphi\leftrightarrow\psi=(\varphi\to\psi)\wedge(\psi\to\varphi)$. 

The consequence relation of CPL is denoted by $\vdash$ and hence $\vdash\varphi$ means that $\varphi$ is a {\em theorem}. If $\mathscr{T}$ is a set of formulas and $\varphi$ a formula, $\mathscr{T}\vdash\varphi$ means that $\varphi$ is provable from $\mathscr{T}$ within CPL.  
A set $\mathscr{T}$ of  formulas of $\mathcal{L}$ is said to be {\em deductively closed} (or a {\em theory} in some textbooks), provided that $\varphi\in \mathscr{T}$ iff $\mathscr{T}\vdash\varphi$.
Two formulas $\varphi_1$ and $\varphi_2$ are said to be {\em logically equivalent} or {\em equal up to logical equivalence} whenever $\vdash \varphi_1\leftrightarrow\varphi_2$.

From the semantic point of view, every formula $\varphi$ can be identified with the set $\mymod{\varphi}$ of its models, i.e., those logical valuations $\omega:\mathcal{L}\to\{0,1\}$ such that $\omega(\varphi)=1$, also written 
 {$\omega\models\varphi$.}
If we denote by $\Omega_{\mathcal{L}}$  the set of logical valuations for $\mathcal{L}$ (or simply $\Omega$ when the language $\mathcal{L}$ is clear by the context), the basic algebraic structure that is needed to interpret formulas in the above sense is the finite Boolean algebra of subsets of $\Omega$, $\POm=(2^\Omega, \cap, \cup, ^c, \emptyset, \Omega)$ with the usual set-theoretic operations of intersection, union and complementation, respectively. For any pair of formulas $\varphi$ and $\psi$, we write $\mymod{\varphi}\subseteq\mymod{\psi}$ to denote that the models of $\varphi$ are included into those of $\psi$. Notice that $\varphi$ and $\psi$ are logically equivalent if and only if $\mymod{\varphi}=\mymod{\psi}$.  The set of consistent formulas will be denoted by $\Lang^\ast$.

Although the results we are going to recall below hold in general for finite Boolean algebras, that we assume the reader to be familiar with, we will henceforth only state them for the case of $\POm$. 

Let us start recalling that a non-empty  {proper} subset $F$ of $\POm$ is a {\em filter} if (i) $\mymod{\varphi},\mymod{\psi}\in F$ implies $\mymod{\varphi}\cap \mymod{\psi}\in F$ (conjunctive closure), and (ii) $\mymod{\varphi}\in F$ and $ \mymod{\varphi}\subseteq \mymod{\psi}$, implies that $\mymod{\psi}\in F$ (upward closure). 
Notice that (ii) implies that $\mymod\top\in F$ for every filter $F$, since filters are non-empty.  {The fact that $F$ is a proper subset of $\POm$ implies that $\mymod\bot\not\in F$.}

 {

Filters of $\POm$ and deductively closed sets of formulas (theories) from $\mathcal{L}$ are tightly linked.
Actually, one can define filters on $\Lang$ in the following way: a subset $\lfilt$ of $\Lang$ is an  $\Lang$-filter if $F=\Set{\mymod{\varphi}\mid \varphi\in\lfilt}$ is a filter on $\POm$.

It is easy to see that $\Lang$-filters $\lfilt$ correspond to consistent theories, that is, subsets of formulas of $\Lang$ which are logically closed and consistent. 
Thus, an $\Lang$-filter $\lfilt$ can be characterized syntactically by: (i) if $\varphi,\psi\in\lfilt$ then $(\varphi\wedge\psi)\in\lfilt$; (ii)
$\varphi\in\lfilt$ and $\vdash(\varphi\implica\psi)$ then $\psi\in\lfilt$; and (iii) $\bot\not\in \lfilt$.
That is to say, an $\Lang$-filter $\lfilt$ is a consistent and deductively closed theory.

Moreover, let $F$ be a filter of $\free$ and consider the set 
$
\mathscr{F}=\{\varphi\in \mathcal{L}\mid \mymod\varphi\in F\}.
$
Then $\mathscr{F}$ is deductively closed and consistent. 
Conversely, let $\mathscr{T}$ be a consistent theory (a consistent deductively closed set of formulas) of $\mathcal{L}$ and consider 
$
F=\{\mymod{\varphi}\in \free\mid \varphi\in \mathscr{T}\}.
$
As we already observed above, $F$ is a filter of $\free$.
Note that by this correspondence, we identify $\Lang$-filters  (consistent theories) with filters of $\free$.

For a later use, let us recall the following.
\begin{fact}\label{factPrincFilt} For every filter $F$ of $\POm$, there exists a unique (up to logical equivalence) $\psi \in {\cal L}$ such that $F=\{\mymod{\varphi}\in \POm\mid \mymod{\psi}\subseteq\mymod{\varphi}\}$. 
\end{fact}
In such a case, we  say that $F$ is {\em principally generated by $\mymod{\psi}$} or that $\mymod\psi$ {\em generates} $F$, and we will write $F={\uparrow}\mymod{\psi}$. 

A theory $\mathscr{T}$ is {\em maximally consistent} if for all $\varphi\in \mathcal{L}$ either $\varphi\in \mathscr{T}$ or $\neg\varphi\in \mathscr{T}$. 
{\em Ultrafilters} are filters that are maximal with respect to the usual set-theoretic inclusion. It is well known that  maximally consistent theories and ultrafilters are in correspondence \cite{CL2000}. 
For a later use, let us 
recall that if $U$ is an ultrafilter of $\free$ then there exists a unique 
valuation $\omega\in \Omega$ such that $U=\{\mymod{\varphi}\in \free\mid \omega\in\mymod\varphi)\}$.

To close this section, and following \cite{Paris}, let us recall that a probability distribution on $\Omega$ is a mapping $P:\Omega\to[0,1]$ such that $\sum_{\omega\in \Omega}P(\omega)=1$,  and extends to subsets of $\Omega$ by additivity, i.e.\ by letting $P(S) = \sum_{\omega \in S} P(\omega)$ for all non-empty $S \subseteq \Omega$ and  $P(\emptyset) = 0$. A probability distribution is said to be {\em positive} if $P(\omega)>0$ for all $\omega\in \Omega$.  For every formula $\varphi\in \mathcal{L}$, we will hence write $P(\varphi)$ for $\sum_{\omega\in \mymod{\varphi}}P(\omega)$. We henceforth  assume probability functions to be positive.

\section{Lockean belief sets that are deductively closed}\label{sec:3}
In this section we will present a first characterization for Lockean belief sets that are deductively closed in terms of suitable probability functions. The main result will be given by selecting those probability functions $P$ for which there exists a parameter $\lambda>1/2$ such that $\mathscr{B}_{\lambda,P}$ is deductively closed. Due to the identification between theories of $\mathcal{L}$ and filters of $\POm$ we mentioned in the above section,  we will consider Lockean belief sets as subsets of $\POm$ and we will  equivalently say that $\Blp$ is deductively closed or that $\Blp$ is a filter without risk of ambiguity.

\begin{remark} Let us notice that the monotonicity of probability functions implies that every Lockean belief set is upward closed. Indeed, if $\psi\in \Blp$ and $\mymod{\psi}\subseteq\mymod{\varphi}$, then $P(\varphi)\geq P(\psi)\geq\lambda$, whence $\varphi\in \Blp$. Therefore, in order for $\Blp$ to be deductively closed, it is enough that it satisfies conjunctive closure.  
\end{remark}

The first result 
introduces a notational convention that will be often used throughout the paper.
\begin{lemma}\label{lemma1}
    For a positive probability $P$ on $\free$ and $\lambda>1/2$, there exist a minimal $\lambda_m$ and a maximal $\lambda_M$ such that $\lambda_M=P(\psi)$ for some $\psi\in \free$, 
$\lambda_m < \lambda\leq\lambda_M$ and
    $\mathscr{B}_{\lambda',P} = \mathscr{B}_{\lambda,P}$ for every $\lambda' \in (\lambda_m, \lambda_M]$.
\end{lemma}
\begin{proof}
Since $\mathcal{L}$ is finite, $\POm$ is finite as well and the image of every probability $P$ on $\POm$ is a finite subset of $[0,1]$ of increasing values, say $\{q_0=0,q_1,q_2,\ldots,q_{t-1}, q_t=1\}$. Thus, for every $\lambda>1/2$, let $0 < i\leq t-1$ such that $q_i <  \lambda\leq q_{i+1}$ and call $\lambda_m=\max\{1/2, q_i\}$ and $\lambda_M=q_{i+1}$. Then, $\psi\in \mathscr{B}_{\lambda,P}$ iff $P(\psi)\geq\lambda$ iff $P(\psi)\geq \lambda'$ for every $\lambda' \in (\lambda_m, \lambda_M)$. In other words, $\mathscr{B}_{\lambda',P} = \mathscr{B}_{\lambda,P}$ for every $\lambda' \in (\lambda_m, \lambda_M]$ and the claim is settled.\ep
\end{proof}


We say that a belief set $\mathscr{B}_{\lambda,P}$ is {\em trivial} when  $\Blp={\uparrow}\mymod\top$. Note that if $P$ is a positive probability and $\Blp$ is not trivial, then $\lambda_M<1$.

An immediate consequence of the above lemma is the following easy remark.

\begin{remark}\label{rem33}
Let $P$ and $\lambda$ be such that $\Blp$ is deductively closed, i.e. $\Blp$ is a filter of $\POm$. Let $\psi$ be such that $\Blp={\uparrow}\mymod{\psi}$. Then, although we cannot ensure that $P(\psi) = \lambda$, we know that $P(\psi) = \lambda_M$, and by the above lemma, $\Blp=\mathscr{B}_{\lambda_M, P}$. 
\end{remark}

Next we establish a result useful in the sequel.

{
\begin{lemma}\label{lemmaNew}

If $\Blp$ is a filter and $P(\psi)=\lambda$ then $\Blp={\uparrow}\mymod\psi$.
\end{lemma}
\begin{proof}
  Let $\varphi$ be such that $\mymod\varphi$ is a generator of the filter $\Blp$.  
 Since  $P(\psi)=\lambda$,  $\psi\in \Blp$. Thus, $\varphi\vdash\psi$, i.e.,
$\mymod\varphi\subseteq\mymod\psi$.  We claim that $\mymod\psi\subseteq \mymod\varphi$ as well.
If not, there exists $\omega$ such that  $\omega\in\mymod\psi$ and $\omega\not \in\mymod\varphi $. Thus, $P(\psi)\geq  P(\varphi)+ P(\omega)> \lambda$, a contradiction.\footnote{ Since $P$ is positive by our general assumption.} Then we have $\mymod\psi=\mymod\varphi$. Therefore, $\Blp={\uparrow}\mymod\psi$. \qed
\end{proof}
}


Let us start our analysis with a first result that collects some basic, yet quite interesting, facts about probabilities and deductively closed sets of formulas, and the proof of which can be found in the Appendix.

\begin{proposition} The following conditions hold:     \vspace{.1cm}

   \noindent(1) Let $P$ and $\lambda$ be such that $\mathscr{B}_{\lambda,P}$ is a filter of $\free$. If $P$ is a homomorphism of $\free$ to $\{0,1\}$, then  $\mathscr{B}_{\lambda,P}$ is maximal  (and in that case $\lambda_M=1$).
    \vspace{.1cm}

    \noindent(2)  $\mathscr{B}_{\lambda,P}$ is a maximal filter iff there is $\omega\in \Omega$ such that $P(\omega)\geq\lambda$. 
    \vspace{.1cm}
    
    \noindent (3) For a positive probability $P$,  {$\mathscr{B}_{\lambda,P}={\uparrow}\mymod\top$} iff $\mathscr{B}_{\lambda,P}=\{\varphi\mid P(\varphi)=1\}$.
\end{proposition}

The above result then highlights the following elementary facts:
\begin{enumerate}
    \item If $P$ is a homomorphism of $\free$ to $\{0,1\}$ then $\mathscr{B}_{\lambda,P}$ is a maximal filter of $\free$ and $\lambda=1$, necessarily. Thus, from what we pointed out in  Section \ref{sec:2}, $\mathscr{B}_{\lambda,P}$ corresponds to a maximally consistent  theory. Moreover, all maximally consistent theories are described in this way.
    \item If $P$ is positive and $\lambda=1$, then again $\mathscr{B}_{\lambda,P}$ is deductively closed, but it corresponds to the trivial theory containing only logical theorems.
\end{enumerate}
It now remains to study the general situation in which $\Blp$ is a proper filter of $\free$ that is neither maximal, nor a singleton. 

Given what we recalled in the above section, for $P$ and $\lambda$ as above, $\mathscr{B}_{\lambda,P}$ is deductively closed if and only if $\Blp$ is a filter of $\free$ and hence, by Fact \ref{factPrincFilt},  there exists a unique $\psi\in \mathcal{L}$ such that $\Blp={\uparrow}\mymod{\psi}$. 

In general, we can prove the following result that provides  a first characterization of those Lockean belief sets that are deductively closed.  In the statement of the next result, and henceforth, we will adopt the notation of Lemma \ref{lemma1}.

\begin{theorem}\label{propFil1}
    For a positive probability $P:\free\to[0,1]$ and $\lambda>1/2$, $\Blp$ is deductively closed iff there exists $\psi\in \Blp$ such that $P(\psi)=\lambda_M$ and $1-\lambda_M<\min\{P(\omega)\mid \omega\in\mymod{\psi}\}$ and, in such a case, $\Blp={\uparrow}\mymod{\psi}$. 
\end{theorem}
\begin{proof}
 By Lemma \ref{lemma1}, $\Blp=\mathscr{B}_{\lambda_M, P}$, and hence  we will prove the theorem considering $\mathscr{B}_{\lambda_M, P}$ without loss of generality. 
  
    Let us assume that $\mathscr{B}_{\lambda_M, P}$ is deductively closed, and hence a filter of $\POm$,  and let $\mymod{\psi}$ be  its generator, i.e., $\mathscr{B}_{\lambda_M, P}={\uparrow}\mymod{\psi}$.
    Towards a contradiction, assume  there exists $\omega\in \mymod{\psi}$ such that $P(\omega)\leq 1-\lambda_M$. Then, take $\gamma$ such that $\mymod{\gamma}=
    \Omega\setminus\{\omega\}$. Clearly $\mymod\gamma\not\supseteq\mymod\psi$ whence $\gamma$ does not belong to $\mathscr{B}_{\lambda_M, P}={\uparrow}\mymod\psi$, while 
    $$
    P({\gamma})
    \geq1-(1-\lambda_M)=\lambda_M,$$ contradicting the fact that filters are upward closed.

    Conversely, assume that $P({\psi})=\lambda_M$ and $1-\lambda_M<\min\{P(\omega)\mid \omega\in\mymod{\psi}\}$. The latter implies that for all $\delta, \gamma$ such that $\mymod{\delta}\subseteq\mymod{\psi}^c$ and $\mymod{\gamma}\subseteq\mymod{\psi}$, 
    \begin{equation}\label{eqPorpMain}
    P({\delta})\leq 1-\lambda_M<\min\{P(\omega)\mid \omega\in\mymod{\psi}\}\leq P({\gamma}).
    \end{equation}
    Therefore, if $\mymod{\tau}\not\supseteq\mymod{\psi}$, one has $P({\tau})=P(({\tau}\wedge\neg{\psi})\vee({\tau}\wedge{\psi}))=P({\tau}\wedge\neg{\psi})+P({\tau}\wedge{\psi})<P({\psi}\wedge\neg{\tau})+P({\tau}\wedge{\psi})$ by (\ref{eqPorpMain}) because $\mymod{\tau\wedge\neg\psi}\subseteq\mymod{\psi}^c$, while $\mymod{\psi\wedge\neg\tau}\subseteq\mymod{\psi}$. Now,  $P(\psi\wedge\neg \tau)+P({\tau}\wedge{\psi})=P({\psi})=\lambda_M$ and hence $\tau\not\in \mathscr{B}_{\lambda_M, P}$. Equivalently, for all $\rho$, if $\rho\in \mathscr{B}_{\lambda_M, P}$ then $\mymod{\rho}\supseteq\mymod{\psi}$, i.e.,  $\mathscr{B}_{\lambda_M, P}={\uparrow}\mymod{\psi}$ and hence it is a filter. \ep
\end{proof}


\begin{corollary} If $\lambda > 1/2$ and  $\Blp={\uparrow}\mymod{\psi}$, and hence deductively closed, with $|\mymod{\psi}| = n$, then $\lambda_M > n/(n+1)$. Hence $P(\psi) > n/(n+1)$ and $P(\neg \psi) < 1/n$.    
\end{corollary}

\begin{proof} By the above theorem, $P(\psi) = \sum_{\omega\in \mymod{\psi}} P(\omega) > n(1- \lambda_M)$, hence $1 = P(\mymod{\psi}) + P(\mymod{\psi}^c) > n(1-\lambda_M) + (1-\lambda_M) = (n+1)(1-\lambda_M)$,  and it follows that $\lambda_M > n/(n+1)$. \ep
\end{proof}


The intuition behind the previous  results is that, for a positive probability $P$ and a threshold $\lambda$ to define a deductively closed set of formulas $\Blp$ one has to ensure:
\begin{itemize}
\item the existence of a formula $\psi$ such that $P(\psi)=\lambda$ (or $P(\psi)$ is greater than, yet sufficiently close to $\lambda$);
\item that $\psi$ is unique and this property can be ensured by requiring that the probability distribution that gives $P(\psi)=\lambda$ is sufficiently low on $\neg\psi$ as required in the hypothesis of Proposition \ref{propFil1};
\item as a consequence, for instance, if $P$ is the counting measure on $\free$ (i.e., it comes from the uniform distribution on $\Omega$), then for no $\lambda$ one has that $\Blp$ is deductively closed, unless  {$\lambda>\frac{n-1}n$, where $n$ is the cardinal of $\free$}. 
\end{itemize}
So, examples of probability distributions that, on the other hand, ensure that $\Blp$ is deductively closed are those $P$ for which there exists  $\omega\in \Omega$ such that $P(\omega)>\sum_{\omega'\neq \omega}P(\omega')$. In this case if $P(\omega)>1/2$ then $\Blp$ is a maximal filter for $\lambda=P(\omega)$.

\section{Step probabilities: a second characterization result}\label{sec:4}
The characterization provided in Theorem \ref{propFil1} will be made more clear in the following Theorem \ref{ThmChar2}. Let us begin by defining a class of probability functions that will play a central role in what follows.
\begin{definition}
    Let $P$ be a  probability function on $\POm$ and let $\omega\in \Omega$. Then $P$ is said to have an {\em $\omega$-step} (or to have {\em a step at $\omega$}), provided that it satisfies 
    \begin{itemize}
\item[($\omega$S)] $P(\omega)>\sum_{\omega': P(\omega')<P(\omega)}P(\omega')>0$.
    \end{itemize}
     A probability $P$ is said to be a {\em step probability} if it has an $\omega$-step for some $\omega\in \Omega$. \footnote{The name ``step probability'' is inspired by {\em big-stepped probabilities} introduced in \cite{Ben-Dub-Prad-99} and also considered in \cite{Leitgeb}; these ideas will be briefly recalled in Section \ref{Sec:related} below.}
\end{definition}

A probability function might have more than one step as the following example shows.
\begin{example}\label{exSteps}
Consider $P$ defined on  $\Omega=\{\omega_1,\omega_2,\omega_3,\omega_4\}$ such that $P(\omega_1)=P(\omega_4)=0.05$, $P(\omega_2)=0.3$ and $P(\omega_3)=0.6$. The probability $P$ has the following steps:
\begin{itemize}
    \item A step at $\omega_3$ because $0.6=P(\omega_3)>P(\omega_1)+P(\omega_2)+P(\omega_4)=0.4$;
    \item A step at $\omega_2$ because $0.3=P(\omega_2)>P(\omega_1)+P(\omega_4)=0.1$.   \ep
\end{itemize}
\end{example}

 Let $P$ be a probability that has a step at $\omega$. Call  $\Phi_\omega$ a formula such that

\begin{equation}\label{eqPsiAlpha}
\mymod{\Phi_\omega}=\{\omega'\in \Omega\mid P(\omega')\geq P(\omega)\}.
    \end{equation}
 

Continuing the above Example \ref{exSteps}, we have that
\begin{itemize}
    \item $\mymod{\Phi_{\omega_3}}=\Set{\omega_3}$;
    \item $\mymod{\Phi_{\omega_2}}=\Set{\omega_2, \omega_3}$.
\end{itemize}
\begin{proposition}\label{propAlphaStep}
If $P$ has an $\omega$-step, then $1 >  P(\Phi_\omega)>1/2$.
\end{proposition}
\begin{proof}
   By the very definition of $\omega$-step and $\Phi_\omega$, one has that 
    $$
    P(\Phi_\omega)\geq P(\omega)>cc=\sum_{\omega'\not\in \mymod{\Phi_\omega}}P(\omega')=P(\neg\Phi_\omega).
    $$
    It follows that $P(\Phi_\omega)>P(\neg\Phi_\omega)$,
therefore $P(\Phi_\omega)>1/2$. Moreover, $P$ having an step at $\omega$ also implies  $0 < \sum_{P(\omega')<P(\omega)}P(\omega') = P(\neg \Phi_\omega)$, from where it also follows that $P(\Phi_\omega) < 1$. 
\ep
\end{proof}
Notice that, if $P$ has an $\omega$-step, then $P(\omega)$ necessarily coincides with the minimum value among the $P(\omega')$'s for $\omega'\in\mymod{\Phi_\omega}$. Such an $\omega$ for which $P(\omega)$ reaches that minimum is not supposed to be unique; in any case the next result holds.
\begin{lemma}\label{lemmaAlphaStep}
If $P$ has an $\omega$-step then $P(\omega)=\min\{P(\omega')\mid \omega'\in \mymod{\Phi_\omega}\}$.
\end{lemma}
We can now prove the following characterization result.

\begin{theorem}\label{ThmChar2}
For every positive probability $P$, there exists $ 1 > \lambda>1/2$ such that $\Blp$ is deductively closed and not trivial iff $P$ has an $\omega$-step. 
\end{theorem}
\begin{proof}
Let us begin by showing the ``if'' part. Assume that $P$ has an $\omega$-step. Then, by Proposition \ref{propAlphaStep}, $1 > P(\Phi_\omega)>1/2$. Take $\lambda=P(\Phi_\omega)$. By Lemma \ref{lemmaAlphaStep} and the definition of $\omega$-step one has
$$
\min\{P(\omega')\mid \omega'\in\mymod{\Phi_\omega}\}=P(\omega)>\sum_{P(\omega')<P(\omega)}P(\omega')=P(\neg\Phi_\omega)=1-\lambda.
$$
Therefore, since $\lambda=P(\Phi_\omega)$,  by Lemma \ref{lemma1}, $\lambda=\lambda_M$ and hence $\Blp=\mathscr{B}_{\lambda_M, P}$. By Theorem \ref{propFil1}, we have $\Blp=\uparrow\mymod{\Phi_\omega}$. And so, clearly $\Blp$ is deductively closed and not trivial. 


Conversely, let $ 1 > \lambda>1/2$ and assume that $\Blp$ is deductively closed and not trivial.  By Theorem \ref{propFil1} and the assumption of not triviality of $\Blp$, there exists $\psi$ such that $P(\psi)=\lambda_M<1$ and $\tau=\min\{P(\omega)\mid \omega\in\mymod{\psi}\}>1-\lambda_M=P(\neg\psi)$. Then let $\omega$ in $\Omega$ be where $P$ attains the  value $\tau$. 
\begin{claim}\label{claim1} $\psi=\Phi_\omega$.
\end{claim}
\begin{proof}(of the Claim)
By definition $\mymod{\Phi_\omega}=\{\omega': P(\omega')\geq P(\omega)=\tau\}$ and hence $\mymod\psi\subseteq\mymod{\Phi_\omega}$ because, if $\omega'\in\mymod{\psi}$, then $P(\omega')\geq\tau$. 
Towards a contradiction, assume that $\mymod{\Phi_\omega}\not\subseteq\mymod\psi$.  Then there exists $\omega'$ such that $\omega'\in\mymod{\Phi_\omega}$  and $\omega'\not\in\mymod\psi$. But since $\omega'\in\mymod{\Phi_\omega}$ then $P(\omega')\geq\tau>1-\lambda_M$, and since $\omega'\not\in\mymod\psi$ then $P(\omega') \leq P(\neg\psi) =1-\lambda_M$, contradiction. Hence, the claim is settled. \ep
\end{proof}
Now we go back to the proof of Theorem \ref{ThmChar2} and we prove that $P$ has a step at $\omega$. Indeed, by the above claim, 
$P(\psi)=P(\Phi_\omega)=\lambda_M < 1$ and hence, 
$$
P(\omega)=\tau>1-\lambda_M=P(\neg \psi)= P(\neg\Phi_\omega) =\sum_{P(\omega')<P(\omega)}P(\omega')  > 0.
$$
Thus, $P$ has a step at $\omega$ and this settles the claim. \ep
\end{proof}
The proof of the above theorem  shows that a probability function $P$ with a step at $\omega$ determines the deductively closed set $$
\mathscr{B}_{P(\Phi_\omega), P}=\{\varphi\mid P(\varphi)\geq P(\Phi_\omega)\}
$$ 
 for $\lambda=P(\Phi_\omega)$, where $\Phi_\omega$ is defined as in (\ref{eqPsiAlpha}). 
Indeed,  a probability $P$ determines as many deductively closed sets as steps it has. 
\begin{example}
    Let us consider again the probability $P$ from Example \ref{exSteps}. The two steps at $\omega_3$ and $\omega_2$ determine the deductively closed sets:
    \begin{itemize}
        \item $\Bpar{0.6}{P}=\{\varphi\mid P(\varphi)\geq P(\omega_3)=0.6\}=\{\varphi\mid \omega_3\in\mymod\varphi\}$;
        \item $\Bpar{0.9}{P}=\{\varphi\mid P(\varphi)\geq P(\Set{\omega_2,\omega_3})=0.9\}=\{\varphi\mid \Set{\omega_2,\omega_3}\subset\mymod\varphi\}$.
    \end{itemize}
    Notice that $\Bpar{0.6}{P}$ is maximally consistent, while $\Bpar{0.9}{P}$ is not and in fact $\Bpar{0.9}{P}$ is strictly contained in $\Bpar{0.6}{P}$. 
\end{example}

\section{Minimal change of Lockean belief sets}\label{sec:5}
Now that we have established under which precise conditions a probability $P$ ensures that a Lockean belief set $\Blp$ is deductively closed we  can put forward a first preliminary analysis on revising $\Blp$ by a formula $\psi$.

Given the probabilistic setting of our approach, a natural way to revise a Lockean belief set $\Blp$ by a formula  $\psi$ (and hence getting $\Blp*\psi$) is relying on the conditional probability $P(\cdot\mid \psi)=P_\psi(\cdot)$, and hence defining  $\Blp\ast\psi=\{\varphi\in \mathcal{L}\mid P_\psi(\varphi)\geq\lambda\}$. This method, which has been already considered by several authors, and notably in \cite{Garden86}, does not generally fit with a basic, yet only informally expressible principle  of {\em minimal change} stipulating that, when revising beliefs, one should make the fewest possible changes to the existing belief set while still accommodating the new information.
Our goal  is hence  to revise  a generic Lockean belief set $\Blp$ by a formula $\psi$, in such a way that the obtained set $\Blp\ast\psi$ meets the following desiderata:
\begin{enumerate}
\item If $\psi\in \Blp$ then $\Blp*\psi=\Blp$, that is to say, no revision is produced by those formulas that are already believed.
\item If $\psi\not\in \Blp$, then we include $\psi$ in $\Blp$ by minimally changing its probability; in other words, if a priori $P(\psi)<\lambda$, then a posteriori $P_\psi(\psi)=\lambda$. 
\end{enumerate}

In order to fulfill the above requests, we will define the following.

\begin{definition}
Given a positive probability $P$ on $\Omega$ and $\psi\in \mathcal{L}$ and $\lambda>1/2$, the {\em revised probability of $P$ by $\psi$ and $\lambda$} is the function $R^{\lambda}_\psi$  defined on $\Omega$ by putting: 
\begin{equation}\label{eqRevFunction}
R^{\lambda}_\psi(\omega)=\left\{
\begin{array}{ll}
P(\omega)\cdot \max\left\{1, \frac{\lambda}{P(\psi)}\right\}&\mbox{ if }\omega\in \mymod{\psi}\\
P(\omega)\cdot\min\left\{1,\frac{1-\lambda}{P(\neg \psi)}\right\} & \mbox{ if }\omega\not\in \mymod{\psi}
\end{array}
\right.
\end{equation}
\end{definition}

Notice that the above definition  accommodates the above desideratum 1 because, if $\psi\in \Blp$, then $P(\psi)\geq \lambda$ and hence $\lambda/ P(\psi)\leq1$. Thus $R^{\lambda}_\psi(\omega)=P(\omega)$ for all $\omega\in \mymod\psi$ and hence $R^{\lambda}_\psi(\psi)=P(\psi)\geq \lambda$. It also satisfies  desideratum 2 because, if $P(\psi)<\lambda$, $R^{\lambda}_\psi(\psi)=\lambda$ as desired.
 
Also,  regardless whether $P(\psi) < \lambda$ or $P(\psi) \geq \lambda$, one can easily check that $\sum_{\omega\in \Omega}R^{\lambda}_\psi(\omega)= 1$.  For instance, when $P(\psi)<\lambda$,  $R^\lambda_\psi(\omega)=P(\omega)\cdot \frac{\lambda}{P(\psi)}$ for $\omega\in \mymod{\psi}$ and $R^\lambda_\psi(\omega)=P(\omega)\cdot\frac{1-\lambda}{P(\neg\psi)}$ for $\omega\not\in\mymod{\psi}$. Thus, $\sum_{\omega\in \Omega}R^\lambda_\psi(\omega)=\sum_{\omega\in \mymod{\psi}}R^\lambda_\psi(\omega)+\sum_{\omega\in \mymod{\neg\psi}}R^\lambda_\psi(\omega)=\sum_{\omega\in \mymod{\psi}}P(\omega)\cdot\frac{\lambda}{P(\psi)}+\sum_{\omega\in \mymod{\neg\psi}}P(\omega)\cdot\frac{1-\lambda}{P(\neg\psi)}=P(\psi)\cdot\frac{\lambda}{P(\psi)}+P(\neg\psi)\cdot\frac{1-\lambda}{P(\neg\psi)}=\lambda+1-\lambda=1$. It follows that 
$R^{\lambda}_\psi$ is a probability distribution on $\Omega$ that extends to a probability function on $\mathcal{L}$ which we will indicate by the same symbol. Interestingly, we can prove the following result that  connects our revised probability function $R^{\lambda}_\psi$ with the well-known {\em Jeffrey conditionalization}, see \cite{Jeffrey} (see also \cite{CR19} for an application  to belief revision): for every positive probability $P$ on $\mathcal{L}$, for every $\varphi,\psi\in \mathcal{L}$ with $0<P(\psi)<1$,
and for every $\lambda\in [0,1]$, the Jeffrey conditionalization of $\varphi$ given $\psi$ is defined as
$$
P_\psi^\lambda(\varphi)=\lambda P(\varphi\mid\psi)+(1-\lambda)P(\varphi\mid \neg \psi)=\lambda \frac{P(\varphi\wedge\psi)}{P(\psi)}+(1-\lambda)\frac{P(\varphi\wedge\neg\psi)}{P(\neg\psi)}
$$
Observe that, when $\lambda=1$ the above expression recovers the usual definition of conditional probability. That is to say, $P_\psi^1(\varphi)=P(\varphi\mid\psi)$.

 \begin{proposition}\label{RandJ}
 For every probability $P$, $\lambda > 1/2$ and $\psi$ such that $P(\psi)>0$, we have 
 for all $\varphi\in \mathcal{L}$, 
 $$
 R^{\lambda}_\psi(\varphi)= 
 \left \{
 \begin{array}{ll}
 P_\psi^\lambda(\varphi), & \mbox{if } P(\psi) \leq \lambda \\
 P(\varphi), & \mbox{otherwise.}
 \end{array}
 \right.
 $$
Thus, when $P(\psi) \leq \lambda$, $R^{\lambda}_\psi$ is  the Jeffrey conditionalization of $P$ by $\psi$.
 \end{proposition} 
\begin{proof}
 We begin by assuming that $P(\psi) \leq \lambda$, so that $\max\{1,\frac{\lambda}{P(\psi)}\}=\frac{\lambda}{P(\psi)}$ and $\min\{1, \frac{1-\lambda}{P(\neg\psi)}\}=\frac{1-\lambda}{P(\neg\psi)}$. Then we have 
$R^{\lambda}_\psi(\varphi) = R^{\lambda}_\psi(\varphi \land \psi) + R^{\lambda}_\psi(\varphi \land \neg \psi) = P(\varphi \land \psi) \cdot \frac{\lambda}{P(\psi)} + P(\varphi \land \neg\psi) \cdot \frac{1-\lambda}{P(\neg\psi)} = \lambda P(\varphi \mid \psi) + (1-\lambda)P(\varphi \mid \neg \psi)$.

Conversely, if $P(\psi) > \lambda$, $\max\{1,\frac{\lambda}{P(\psi)}\}=\min\{1, \frac{1-\lambda}{P(\neg\psi)}\}=1$. Then,  $R^{\lambda}_\psi(\varphi) = R^{\lambda}_\psi(\varphi \land \psi) + R^{\lambda}_\psi(\varphi \land \neg \psi) = P(\varphi \land \psi)\cdot 1 +  P(\varphi \land \neg \psi)\cdot 1 = P(\varphi)$. 
\qed
 \end{proof}


Using the revised probability $R^\lambda_\psi$, we now formally introduce our revision operator for Lockean belief sets.

Actually, the present setting may be read within the framework introduced in  \cite{SKP22} by defining an epistemic space  $\mathcal{E}=(\mathscr{P},B_\lambda)$ where $\frac{1}{2}<\lambda<1$, $\mathscr{P}$ are the positive probabilities on $\POm$ and for every $P\in \mathscr{P}$ we put $B_\lambda(P)=\Blp$. Then, on this epistemic space, we define an operator $\circ:\mathscr{P}\times\Lang^*\rightarrow\mathscr{P}$ by putting $P\circ \psi=R^\lambda_\psi$.
This operator induces a revision operator $*_{ml}$ at the level of beliefs. More precisely we have the following definition:

\begin{definition}
Given a Lockean belief set $\Blp$, and a proposition $\psi\in\Lang^*$, the {\em minimal Lockean} revision of $\Blp$ by $\psi$ is defined as: 
$$
\Blp *_{ml} \psi=\{ \varphi \mid R^\lambda_\psi(\varphi) \geq \lambda \} = \mathscr{B}_{\lambda, R^\lambda_\psi}
$$ 
\end{definition}
Note that $\Blp *_{ml} \psi=B_\lambda(P\circ\psi)$.

We are now ready to put together the notions and  results provided so far  to characterize under which conditions we can revise a Lockean belief set $\Blp$  in such a way that the obtained $\Blp *_{ml}\psi$ is deductively closed. 
Next result fully describes the  situations in which $\Blp *_{ml} \psi$ is deductively closed. 


{
\begin{theorem} \label{teotau} Let $\Blp$ be a Lockean belief set, $\psi$ a proposition and  $\tau = \min\{P(\omega) \mid \omega \in \mymod\psi\}$. Then:
\begin{itemize}
\item[(i)]  $\Blp *_{ml} \psi$ is deductively closed whenever $P(\psi)<\frac{\tau\lambda}{1-\lambda}$. Moreover, in that case, $\Blp *_{ml} \psi={\uparrow}\mymod{\psi}$
\item[(ii)] Conversely, assuming $\psi \not\in \Blp$, 
if $\Blp *_{ml} \psi$ is deductively closed then $P(\psi)<\frac{\tau\lambda}{1-\lambda}$. 
\end{itemize}
\end{theorem}
}

\begin{proof}
$(i)$ 
According to Theorem \ref{propFil1}, in order to show that $\mathscr{B}_{\lambda, R^\lambda_\psi}=\Blp *_{ml} \psi$ is deductively closed, it is enough to find $\chi\in\mathcal{L}$ such that $R^\lambda_\psi(\chi) = \lambda$ and that $1-\lambda < R^\lambda_\psi(\omega)$ for all $\omega\in \mymod\chi$. Notice that in that case $\lambda=\lambda_M$ relative to the probability $R^\lambda_\psi$.

Let us take $\chi = \psi$ itself. Then $R^\lambda_\psi(\psi) = \lambda$ holds by definition, so it remains to check that  $1-\lambda < R^\lambda_\psi(\omega)$ for all $\omega \in \mymod\psi$. But, again by the definition of $R^\lambda_\psi$, if $\omega \in \mymod\psi$ then $R^\lambda_\psi(\omega) = \lambda  P(\omega)/P(\psi)$. Hence,  $R^\lambda_\psi(\omega)=\lambda  P(\omega)/P(\psi) > 1-\lambda$ iff $P(\psi)<\frac{P(\omega)\lambda}{1-\lambda}$ for all $\omega\in \mymod\psi$ and hence iff $P(\psi)<\frac{\tau\lambda}{1-\lambda}$.
\vspace{.1cm}

{$(ii)$ If $\psi \not\in \Blp$ we have $P(\psi)<\lambda$  and $R_\psi^\lambda(\psi)=\lambda$. Since $\Blp *_{ml} \psi$ is deductively closed, by Lemma \ref{lemmaNew}, $\mymod\psi$ generates $\Blp *_{ml} \psi$ and 
by Theorem~\ref{propFil1}, $1-\lambda<\tau$. Thus $\frac{\tau}{1-\lambda}> 1$ and hence $P(\psi)<\lambda<\frac{\tau}{1-\lambda}\cdot\lambda$. \qed}
\end{proof}
Let us observe that in $(ii)$ above, the condition $\psi \not \in \Blp$ is in fact necessary as the following simple example shows.  Take $\Omega=\Set{\omega_1,\dots,\omega_4}$ and  define $P$ by
$P(\omega_1)=P(\omega_2)=0.35$, $P(\omega_3)=P(\omega_4)=0.15$. Put $\lambda=0.7$.
It is easy to see that $\Blp$ is closed. Let $\psi$ be a formula such that $\mymod{\psi}=\Set{\omega_1,\omega_2,\omega_3}$, i.e., with $P(\psi) = 0.85$. Thus,  ${\Blp}\ast_{ml}\psi=\Blp$ is closed but 
$P(\omega_3)<P(\psi)\frac{1-\lambda}{\lambda}$.

{
Now we briefly examine  which of AGM basic postulates  for revision (K*1-K*6) from \cite{AGM}  are satisfied in our framework. More detailed proofs can be found in the Appendix.
\begin{description}
\item[K*1 (Closure):] {\em $\Blp*\psi$ is logically closed}. This is not generally satisfied. However, it is satisfied if either $\psi\in \Blp$ and $\Blp$ is closed or $\psi\not \in \Blp$ and  $\psi$ satisfies the condition (i) of Theorem~\ref{teotau}.

\item[K*2 (Success):] $\psi\in\Blp*\psi$ . This holds by definition of $R^\lambda_\psi$.

\item[K*3 (Inclusion):] $\Blp*\psi\subseteq Cn(\Blp \cup\Set{\psi})$. This postulate holds.  Essentially, this is due to the fact that  if $\varphi\in \Blp*\psi$ then $(\psi\rightarrow\varphi)\in \Blp$

\item[K*4 (Preservation):] {\em If $\Blp\not\vdash \neg\psi$ then $\Blp\subseteq \Blp*\psi$}. This postulate is not satisfied in general (see the Appendix for a counterexample).

\item[K*5 (Consistency):] {\em If $\psi$ is consistent then $\Blp*\psi$ is consistent}.
This postulate is not satisfied in general (see the Appendix for a counterexample). However, if $\psi$ satisfies the conditions of Theorem~\ref{teotau}.(i), the postulate holds.

\item[K*6 (Extensionality):] {\em If $\vdash \psi\leftrightarrow\varphi$ then $\Blp*\psi=\Blp*\varphi$}. This postulate clearly holds because the revision  is defined semantically.
\end{description}
}

Finally, let us  observe that  the function $R_\psi^\lambda$, used for our revision is somehow minimal in the sense that $R^\lambda_\psi$ changes less drastically the a priori $P$ than the usual  conditioning by
$\psi$ when $\lambda<1$ and $P(\psi)<\lambda$. 
One of the most basic distances $d$ between probabilities  is defined by 
\begin{equation}
    d(P,P')=\sum_{\omega\in\Omega}|P(\omega)-P'(\omega)|.
\end{equation}
The following, whose proof appears detailed in the Appendix, holds.
\begin{fact}\label{fact1}
According to the above distance function $d$, $R^\lambda_\psi$ is closer to $P$ than $P(\cdot\mid\psi)$, the simple conditioning.    
\end{fact}

It should be noted that the distance $d$ is too rough to distinguish $R^\lambda_\psi$ among the set of those distributions satisfying 
that the probability of $\psi$ is equal to $\lambda$. 
However, there exists a measure of divergence 
(which is not a distance) for which $R^\lambda_\psi$ is minimal with respect to distributions satisfying 
that the probability of $\psi$ is equal to $\lambda$. This is the {\em Kullback-Leibler divergence} \cite{Jac,PP25} (also called relative entropy), defined as follows:
\[
\dkl(P',P)=\sum_{\omega\in\Omega}P'(\omega)\log(P'(\omega)/P(\omega))
\]
Using the Lagrange multipliers method, one can easily prove the next fact, whose proof is again detailed in the Appendix.
\begin{proposition}\label{fact2}
    $R^\lambda_\psi$ minimizes the Kullback-Leibler divergence among the distributions $P'$ such that 
$P'(\psi)=\lambda$.
\end{proposition}

\section{Related work}
\label{Sec:related}

There are several definitively relevant works in the literature which are related to the approach proposed in this paper in the sense of using standard (non-infinitesimal) probabilistic semantics to model the notion of belief. 

 In the context of conditional knowledge bases and the non-monotonic reasoning System P, the authors in \cite{Ben-Dub-Prad-99}  look for some standard probabilistic semantics for defaults ``generally, if $\varphi$ then $\psi$'' of the kind $P(\psi \mid \varphi) > 1/2$, or equivalently  $P(\psi \mid \varphi) >  P(\neg \psi \mid \varphi)$. Since  this general semantics does not satisfy  the OR postulate of System P and the corresponding sets of (conditional) beliefs $\{\psi \mid  P(\psi \mid \varphi) > 1/2\}$ are not deductively closed, they consider a particular subclass of probability distributions, namely Snow's {\em atomic bound systems} \cite{Snow}, also called big-stepped probabilities by these authors. More precisely, a big-stepped probability on a (finite) algebra set ${\bf A} = 2^\Omega$ is a probability function $P: {\bf A} \to [0, 1]$ such that (i) $P(w) > 0$ for all $w \in \Omega$, (ii) it induces a linear ordering $>$ on $\Omega$, whereby $w > w'$ if and only if $P(w) > P(w')$, and (iii)  for each interpretation $w$, $P(w) > \sum_{w': w > w'} P(w')$. From this definition it is clear that a big-stepped probability is a stronger notion than an $\omega$-step probability. 

Big-stepped probabilities are closely related to probabilities which are {\em acceptance functions}. According to \cite{DubPrad-95}, an acceptance function is a mapping  $g: 2^\Omega \to [0, 1]$) satisfying:  
\vspace{0.1cm}

(AC1) if $g(\varphi) > g(\neg \varphi)$ and $\varphi \models \psi$ then $g(\psi) > g(\neg \psi)$, and 

(AC2)  if $g(\varphi) > g(\neg \varphi)$ and  $g(\psi) > g(\neg \psi)$ then  $g(\varphi \land \psi) > g(\neg (\varphi \land \psi))$. \vspace{0.1cm}

\noindent Then $\varphi$ is an accepted belief when $g(\varphi) > g(\neg \varphi)$. Indeed, if $g$ is an acceptance function, then  $\{\varphi \mid g(\varphi) > g(\neg \varphi) \}$ is a belief set that is deductively closed. The authors show in \cite{DubPrad-95} that the only (positive) probabilities $P$ that are acceptance functions are those that satisfy  either: 
\vspace{0.1cm}

(i)  $P(\omega) > 1/2$ for some  $\omega \in \Omega$, or 

(ii)  there exist $\omega,\omega' \in \Omega$ such that $P(\omega) = P(\omega') = 1/2$. \vspace{0.1cm}

\noindent It is clear that big-stepped probabilities are both $\omega$-step probabilities (for some $\omega$) and (probability) acceptance functions,  
but the converse does not hold (one has to require that $P$ induces a linear order on $\Omega$). 
On the other hand, it is easy to check that non-trivial\footnote{A  probability acceptance function $P$ is called {\em trivial} when  $P(\varphi) > 1/2 $ iff $\models \varphi$} probability acceptance functions are exactly those $\omega$-step probability functions for which there exists $\lambda > 1/2$ such that  $\Blp$ is (deductively closed and) a maximal filter.

Another closely related work is that of Leitgeb \cite{Leitgeb}, where the author considers a weakened version of the Lockean thesis, namely, given a probability P and a threshold $\lambda$, the corresponding beliefs should necessarily have a probability higher than $r$, but this is not a sufficient condition to become a belief. The corner notion at work in his proposal is the concept of $P$-stable$^\lambda$ set. Adapted to our framework, a proposition $\psi$ is $P$-stable$^\lambda$ if for any proposition $\varphi$ such that $\psi \land \varphi \not\models \bot$ and $P(\varphi) > 0$, it holds that $P(\psi \mid \varphi) > \lambda$.\footnote{\label{aa} The author shows that this condition is equivalent to require that, for any model $\omega$  of $\psi$, it holds that $P(\omega) > \frac{\lambda}{1-\lambda} P(\neg \psi)$.} 
What $P$-stability$^\lambda$ requires is that if some proposition is consistent with all the other beliefs, then conditionalizing on such a proposition should not decrease the probability
of any believed proposition below the threshold $\lambda$. 
But, as the author notices, not every believed proposition has to be P-stable$^\lambda$. In
fact, it is shown in \cite{Leitgeb} that the set of rationality postulates which he proposes a set of beliefs should satisfy (one of them being to be closed under deduction) entails that P-stability$^\lambda$ is only required for the 
strongest believed proposition, which is the conjunction of all believed propositions. 
Obviously,  P-stable$^\lambda$ sets are a stronger notion than closed Lockean belief sets. Indeed, using the condition expressed in Footnote \ref{aa}, one can check that if $\Blp$ is closed and $\psi$ is its strongest belief, then $P(\neg\psi) = 1-\lambda$, while $\psi$ is $P$-stable$^\lambda$ only when $P(\neg\psi) < (1-\lambda). \frac{\tau}{\lambda} \leq 1-\lambda$, where $\tau = \min\{P(\omega) \mid \omega \in \mymod\psi\}$.

\section{Conclusions and future work}\label{sec:7}

In the first part of the  paper we have provided two characterizations of those  probability functions $P$ on formulas ensuring the existence of suitable thresholds $\lambda$ for which the corresponding  Lockean set of beliefs $\Blp$ is closed under logical deduction. Then, in the second part, we have discussed  how to revise a Lockean belief set in a way compatible with an intuitive principle of minimal change. We have shown that this principle univocally leads to a revision procedure of an a priori given probability by a formula that, in fact, is closely related (but not equal) to Jeffrey conditionalization. 
Finally, we have characterized under which condition a (non necessarily deductively closed) Lockean belief set becomes deductively closed once revised by a suitable formula. 

Regarding our revision operator for Lockean belief sets $*_{ml}$, an interesting line for future work is to investigate the fulfillment of Darwiche-Pearl postulates for iterated revision \cite{DARWICHE19971}, and to find which postulates fully characterize $*_{ml}$ for a given probability and a given threshold. Also we plan to study in more detail the links of our approach with that of Leitgeb based on $P$-stable$^\lambda$ sets \cite{Leitgeb}.  

 \subsubsection*{Acknowledgments} The authors thank the reviewers for their valuable and helpful comments as well as Vanina M. Martinez and Ricardo O. Rodriguez for initial interesting discussions on the topic of the paper.  Flaminio acknowledges partial support from the Spanish project SHORE (PID2022-141529NB-C22) funded by the MCIN/AEI/10.13\-039/501100011033. Godo acknowledges support by the Spanish project LINEXSYS (PID2022-139835NB-C21). Flaminio and Godo also acknowledge partial support from the H2020-MSCA-RISE-2020 project MOSAIC (Grant Agreement number 101007627). 
Ram\'on Pino P\'erez has benefited from the support of the AI Chair BE4musIA of the French National Research Agency
(ANR-20-CHIA-0028) and the support of  I. Bloch's chair in AI (Sorbonne Universit\'e and SCAI).

\bibliographystyle{splncs04}

\newpage

\section*{Appendix}

\subsection*{A: Some proofs}

{\em Proof of Proposition 1.}

(1) If $P$ is a homomorphism to $\{0,1\}$, for all $\lambda> 1/2$, one has that $\Blp=\{\varphi\mid P(\varphi)\geq \lambda\}=\{\varphi\mid P(\varphi)=1\}$. Since $P$ is a homomorphism, $\Blp$ is a maximal filter from what we recalled in the previous section. 
\vspace{.1cm}

\noindent(2) Trivial because if $\alpha$ is such that $\mymod\alpha=\{\omega\}$, then  $\alpha\in \mathscr{B}_{\lambda,P}$ and $\Blp$ is upward closed.
\vspace{.1cm}

\noindent (3) If $P$ is a positive probability, $P(\varphi)=1$ if and only if $\varphi$ is a tautology and hence iff $\varphi\equiv\top$.
\ep

$$
\ast  
$$ 

\noindent {\em Proof of Fact 2.}
The claim can be proved by a simple computation. Assume $\lambda<1$ and $P(\psi)<\lambda$. 
We have
\[
\begin{array}{lcl}
   d(R_\psi,P)  & = & \sum_{\omega\in\Omega}|R_\psi(\omega)-P(\omega)| \\
     &=& \sum_{\omega\in\mymod\psi}|R_\psi(\omega)-P(\omega)|+
\sum_{\omega\in\mymod{\neg\psi}}|R_\psi(\omega)-P(\omega)| \\
     & =& 
     \sum_{\omega\in\mymod\psi}(R_\psi(\omega)-P(\omega))+
\sum_{\omega\in\mymod{\neg\psi}}(P(\omega)-R_\psi(\omega)) \\
& =& 
     \sum_{\omega\in\mymod\psi}R_\psi(\omega)-\sum_{\omega\in\mymod\psi}P(\omega)+
\sum_{\omega\in\mymod{\neg\psi}}P(\omega)-\sum_{\omega\in\mymod{\neg\psi}}R_\psi(\omega)) \\
 & =& \lambda - P(\psi) + P(\neg\psi) -R_\psi(\neg\psi)\\
 & =& \lambda - P(\psi) + 1 - P(\psi) - (1 - \lambda)\\
  & =& 2(\lambda-P(\psi))
\end{array}
\]
with similar calculations, we can check that $d(P(\cdot\mid\psi),P)=2(1-P(\psi))$. 
By our assumptions $(\lambda-P(\psi))<(1-P(\psi))$, that is $d(R_\psi,P)<d(P(\cdot\mid\psi),P)$. \ep

$$
\ast
$$

\noindent {\em Proof of Proposition 4.}
Put
$\mymod{\psi}=\Set{\omega_1,\dots,\omega_k}$ and $\mymod{\neg\psi}=\Set{\omega_{k+1},\dots,\omega_n}$, let us denote the variables $P'(\omega_i)$ by $x_i$, and put
\[
\begin{array}{lcl}
   f(x_1,\dots,x_n)&=&\sum_{i=1}^nx_i\log(x_i/P(\omega_i))\\
   g(x_1,\dots,x_n)&=&\sum_{i=1}^kx_i=\lambda\\
   h(x_1,\dots,x_n)&=&\sum_{i=k+1}^nx_i=1-\lambda.
\end{array}
\]
Now put $$L(x_1,\dots,x_n)=\sum_{i=1}^nx_i\log(x_i/P(\omega_i))-\alpha\left[\sum_{i=1}^kx_i-\lambda\right]-\beta\left[\sum_{i=k+1}^nx_i-1+\lambda\right].$$
Then,

\begin{equation*}
   \frac{\partial L}{\partial x_i}=\left\{
\begin{array}{lcl}
   \log(x_i/P(\omega_i)) +1-\alpha, && \mbox{ if } i\leq k\\
   \log(x_i/P(\omega_i)) +1-\beta, && \mbox{ if } i> k\\
\end{array}
   \right.
\end{equation*}
Equating these partial derivatives to 0, we have

\begin{equation*}
   x_i=\left\{
\begin{array}{lcl}
   P(\omega_i)e^{\alpha-1}, && \mbox{ if } i\leq k\\
   P(\omega_i)e^{\beta-1}, && \mbox{ if } i> k\\
\end{array}
   \right.
\end{equation*}
Finally, using the constraints for functions $g$ and $h$ we get
$e^{\alpha-1}=\lambda/P(\psi)$ and $e^{\beta-1}=(1-\lambda)/P(\neg\psi)$. That is $\rlpsi$ is a minimum for $f$ under the restrictions $g=\lambda$ and $h=1-\lambda$.
\ep

\subsection*{B: Satisfiability of the AGM basic postulates}

{\bf (K*1):} the Closure postulate does not always hold. Note that if $\psi\in\Blp$, we have  $\rlpsi=P$, then $\Blp\ast_{ml}\psi=\Blp$. Now take $\Omega=\Set{\omega_1,\dots,\omega_4}$ and let $P$ be the probability measure defined by $P(\omega_i)=0.25$ for $i=1,\dots,4$. Put  $\lambda=0.75$ and
$\psi_1$ and $\psi_2$ such that $\mymod{\psi_1}=\Set{\omega_1,\omega_2,\omega_3}$ and 
$\mymod{\psi_2}=\Set{\omega_2,\omega_3,\omega_4}$. Clearly,
$\psi_1, \psi_2\in \Blp$, but $\psi_1\wedge \psi_2\not\in\Blp$. Thus,
$\Blp$ is not logically closed. However, $\Blp\ast_{ml}\psi_1=\Blp$, therefore $\Blp*\psi_1$ is not closed.

Note that if $\Blp$ is closed and $\psi\in \Blp$, $\Blp\ast_{ml}\psi$ is closed because $\Blp\ast_{ml}\psi=\Blp$. Note also that if $\psi\not\in \Blp$ but 
$P(\psi)<\frac{\tau\lambda}{1-\lambda}$ then, by Theorem~3(i),
$\Blp*\psi$ is logically closed.
\vspace{.2cm}

\noindent{\bf (K*2):} the Success postulate holds. If $\psi\in\Blp$, $\rlpsi=P$. Thus, $\Blp\ast_{ml}\psi=\Blp$. Therefore $\psi\in\Blp\ast_{ml}\psi$. If  $\psi\not\in\Blp$, $\rlpsi(\psi)=\lambda$.  Thus, $\psi\in\Blrp$, that is  $\psi\in\Blp\ast_{ml}\psi$.
\vspace{.2cm}

\noindent{\bf (K*3):} the Inclusion postulate holds. In order to see that we establish the following claim:

\begin{claim}
    If $\varphi\in \Blp\ast_{ml}\psi$ then $(\psi\rightarrow\varphi)\in \Blp$.
\end{claim}

\begin{proof}
    We have to see  that  $\rlpsi(\varphi) \geq \lambda$ implies $P(\neg \psi \lor \varphi) \geq \lambda$. 

    If $\psi\in \Blp$, we have $\rlpsi=P$. Thus,  $\rlpsi(\varphi) \geq \lambda$ means 
$P(\varphi)\geq\lambda$. Therefore, $P(\neg \psi \lor \varphi)\geq\lambda$. If $\psi\not\in \Blp$, we have $P(\psi)<\lambda$. Moreover,
    observe that
 $P(\neg \psi \lor \varphi)  = P(\neg \psi) + P(\psi \land \varphi)$ and
  $\rlpsi(\neg \psi \lor \varphi)  = \rlpsi(\neg \psi) + \rlpsi(\psi \land \varphi) = (1-\lambda) + \rlpsi(\psi \land \varphi) $. 
  \noindent Therefore, 
   $P(\neg \psi \lor \varphi) - \rlpsi(\neg \psi \lor \varphi) = P(\neg\psi) + P(\psi \land \varphi) - (1-\lambda) - P(\psi \land \varphi)\cdot \lambda/  P(\psi)$
   $= (\lambda - P(\psi))\cdot (1-P(\varphi \mid \psi)) \geq 0$, because the two factors are positive. 
  Hence, $P(\neg \psi \lor \varphi) \geq \rlpsi(\neg \psi \lor \varphi) \geq \rlpsi( \varphi) \geq \lambda$. \qed

\end{proof}

Now, we prove the satisfaction of (K*3). Let $\varphi $ be  a sentence in $\Blp\ast_{ml}\psi$. By the claim, $(\psi\rightarrow\varphi)\in \Blp$. Then by Modus Ponens we have $\varphi\in Cn(\Blp\cup\Set{\psi})$. Thus, $\Blp\ast_{ml}\psi\subseteq Cn(\Blp\cup\Set{\psi})$.

\vspace{.2cm}

\noindent{\bf (K*4):} the Preservation postulate does not hold. 
Let $\Omega=\Set{\omega_1,\dots,\omega_4}$ and let $P$ be the probability measure defined by $P(\omega_1)=1/8$, $P(\omega_2)=P(\omega_3)=1/4$ and $P(\omega_4)=3/8$. Put $\lambda=7/8$. Let $\varphi$ be a formula such that 
$\mymod\varphi=\Set{\omega_2,\omega_3,\omega_4}$. By Theorem~1, $\Blp= {\uparrow}\mymod\varphi$. In particular $\Blp$ is consistent and $\varphi\in\Blp$. Let $\psi$ be a sentence such that $\mymod\psi=\Set{\omega_1,\omega_2}$. Consider $\Blp\ast_{ml}\psi$
This set is by definition $\Blrp$. Note that $\psi$ is consistent with $\Blp$ because $\omega_2\models \Blp$ and $\omega_2\models \psi$. However, 
\[
\rlpsi(\varphi)=\rlpsi(\omega_2)+\rlpsi(\Set{\omega_3,\omega_4})=\frac78\cdot\frac 14 \cdot\left( \frac 38\right)^{-1}+\frac 18=\frac{17}{24}
\]
But $\frac{17}{24}<\frac 78$. Thus, $\varphi\not\in\Blrp$ and therefore, (K*4) is violated.
\vspace{.2cm}

\noindent{\bf (K*5):} the consistency postulate does not hold. The example given to show that (K*1) fails can be used to prove the failure of (K*5): 
$\psi_1$ is consistent but $\Blp\ast_{ml}\psi_1=\Blp$ is not consistent because clearly there is no $\omega\in\Omega$ such that $\omega\models\Blp$. 
However if $\psi$  is consistent and $P(\psi)<\frac{\tau\lambda}{1-\lambda}$ then, by Theorem~3(i),
$\Blp\ast_{ml}\psi$ is consistent (it is a filter).

\vspace{.2cm}

\noindent{\bf (K*6):} the Extensionality postulate holds. Suppose that $\vdash \psi_1\leftrightarrow\psi_2$. Then $\rlpsiuno=\rlpsidos$, therefore $\Blp\ast_{ml}\psi_1=\Blp\ast_{ml}\psi_2$.

\end{document}